\documentclass[letterpaper, 10 pt, conference]{ieeeconf}  

\usepackage{graphicx}
\usepackage{booktabs} 

\usepackage{hyperref}       
\usepackage{url}            
\usepackage{amsfonts}       
\usepackage{amsmath, amssymb, enumerate, bm, capt-of,mathtools}
\usepackage[capitalise]{cleveref}
\usepackage{siunitx}
\usepackage{tabularx}
\usepackage[norelsize,ruled,linesnumbered,noend]{algorithm2e}

\usepackage{tikz, tikzscale,pgfplots}
\usetikzlibrary{patterns}
\usepgfplotslibrary{fillbetween}

\newtheorem{assumption}{Assumption}[section]
\newtheorem{lemma}{Lemma}[section]
\newtheorem{proposition}{Proposition}[section]
\newtheorem{definition}{Definition}[section]
\newtheorem{theorem}{Theorem}[section]

\newtheorem{remark}{Remark}[section]

\DeclareMathOperator*{\argmax}{arg\,max}
\DeclareMathOperator*{\argmin}{arg\,min}

\newcommand*{\state}{\bm x}
\newcommand*{\statet}[1]{\state_{#1}}
\newcommand*{\action}{\bm u}
\newcommand*{\actiont}[1]{\action_{#1}}
\newcommand*{\noise}{\bm \omega}
\newcommand*{\noiset}[1]{\noise_{#1}}

\newcommand*{\stateset}{\mathcal{X}}
\newcommand*{\safeset}{\mathcal{X}_{\mathrm{safe}}}

\newcommand*{\unsafeset}{\mathcal{X}_{\mathrm{unsafe}}}
\newcommand*{\actionset}{\mathcal{U}}
\newcommand*{\noisedist}{\rho}
\newcommand*{\policy}{\bm{\pi}}
\newcommand*{\safepolicy}{\policy_{\mathrm{safe}}}
\newcommand*{\hpolicy}{\bm{\eta}}
\newcommand*{\hpolicyaction}{\bm{n}}

\newcommand*{\dynamics}{\bm{f}^{\star}}
\newcommand*{\model}{\bm{\mu}}
\newcommand*{\ucbound}{\beta\bm{\Sigma}(\state,\action)\hpolicy(\state,\action)}
\newcommand*{\ucboundn}{\beta\bm{\Sigma}(\state,\action)\hpolicyaction}

\newcommand*{\expectation}[2]{\mathbf{E}_{#1}\!\!\left[#2\right]}

\newcommand*{\levelset}[1]{\mathcal{C}_{#1}^{\policy}}

\newcommand{\probability}[1]{\mathbb{P}\left(#1\right)}
\newcommand{\indicator}[1]{\mathbb{I}_{#1}}
\newcommand*{\Vlevelset}[1]{\mathcal{V}_{#1}}

\newcommand*{\return}{R}
\newcommand*{\reward}{r}

  {%
   \par\noindent{\bfseries\upshape Proof of #1\ }%
  }%
  {\hfill\qedsymbol}

\pgfplotsset{width=10\columnwidth /10, compat = 1.13, 
	height = 55\columnwidth /100, grid= major, 
	legend cell align = left, ticklabel style = {font=\scriptsize},
	every axis label/.append style={font=\small},
	legend style = {font=\tiny},title style={yshift=-7pt, font = \small} }

\usepackage{hyperref}

\IEEEoverridecommandlockouts                              

\title{\LARGE \bf
Safe Reinforcement Learning via Confidence-Based Filters
}

\author{Sebastian Curi$^{1*}$, Armin Lederer$^{2*}$, Sandra Hirche$^{2}$, Andreas Krause$^{1}$
\thanks{This work was supported by the European Research Council (ERC) Consolidator Grant 
"Safe data-driven control for human-centric systems (CO-MAN)" under grant 
agreement number 864686 and under the European Unions Horizon
2020 research and innovation program grant agreement No
815943.
A. L. gratefully  acknowledges  financial  support from  the German Academic Scholarship Foundation.
}%
\thanks{$^{1}$Learning \& Adaptive Systems Group, Department of Computer Science, ETH Zurich, Switzerland. Email:
        {\tt\small [sebastian.curi,krausea]@inf.ethz.ch}}%
\thanks{$^{2}$Chair of Information-oriented Control, Department of Electrical and Computer Engineering, Technical University of Munich, Germany. Email:
	{\tt\small [armin.lederer, 
	hirche]@tum.de}}%
\thanks{*These authors contributed equally.}%
}

\begin{document}

\maketitle
\thispagestyle{empty}
\pagestyle{empty}

\setlength{\textfloatsep}{4pt}
\setlength{\floatsep}{4pt}
\setlength{\abovedisplayskip}{4pt}
\setlength{\belowdisplayskip}{4pt}

\begin{abstract}

Ensuring safety is a crucial challenge when deploying reinforcement learning (RL) to real-world systems.
We develop {\em confidence-based safety filters}, a control-theoretic approach for certifying state safety constraints for nominal policies learnt via standard RL techniques, based on probabilistic dynamics models.
Our approach is based on a reformulation of state constraints in terms of cost functions, reducing safety verification to a standard RL task. By exploiting the concept of {\em hallucinating inputs}, we extend this formulation to determine a ``backup’’ policy which is safe for the unknown system with high probability. Finally, the nominal policy is minimally adjusted at every time step during a roll-out towards the backup policy, such that safe recovery can be guaranteed afterwards.
We provide formal safety guarantees, and empirically demonstrate the effectiveness of our approach.

\end{abstract}



\section{INTRODUCTION}
When agents operate autonomously in unknown environments, they need the ability to adapt to new situations. This adaptiveness can be achieved using reinforcement learning~\cite{Sutton2017}, which allows autonomous agents to modify their behavior according to observations of the environment. Such reinforcement learning approaches have been demonstrated to achieve state-of-the-art performance on various problems where high-fidelity simulations are available \cite{Lillicrap2016}, but they cannot be directly applied to real-world autonomous systems because their safe operation must always be guaranteed. This has lead to safety being a major hurdle for the application of reinforcement learning in real-world applications \cite{Dulac-Arnold2019}.

\paragraph*{Related work} Due to this high relevance of safety in reinforcement learning, it has been the focus of a variety of recent approaches (see \cite{Garcia2015,Brunke2021} for surveys). 
A common framework for safe reinforcement learning are constrained Markov decision processes (CMDP) \cite{Altman1999}. In the CMDP setting, constraints are posed on the expected cumulative cost along roll-outs of a policy. This allows to treat the cumulative cost analogously to rewards, such that methods such as trust region policy optimization can be adapted to maintain constraint satisfaction when initialized with a safe policy~\cite{Achiam2017}. When no initially safe policy is known, a Lagrangian relaxation can be used to asymptotically find safe policies~\cite{Paternain2019}. This dual representation of the constrained optimization problem can also be combined with techniques such as upper confidence reinforcement learning to guarantee learning rates, e.g., for linear CMDPs~\cite{Ding2021}. However, constraint violation during training cannot be excluded in general, which often prevents the usage in real-world applications.\looseness=-1

Control theoretic methods consider safety through constraints on the system states, but often severely restrict the allowed policy and system classes. For example, linear quadratic regulators can be learned efficiently under polytopic constraints on the system states as shown by \cite{Dean2018}. The limitation to linear dynamics can be relaxed in the case of deterministic systems by employing model predictive control (MPC) techniques, such that the performance can be iteratively improved~\cite{Rosolia2018}.
To allow active exploration, MPC can be combined with reinforcement learning ideas by ensuring safety using a suitable MPC parameterization~\cite{Koller2018}.
Moreover, Monte-Carlo approximations can be used to deal with the stochastic dynamics  commonly found in reinforcement learning problems \cite{Capone2021}, although this comes at the price of a prohibitive computational complexity.

To achieve the beneficial properties of both control theoretical and reinforcement learning based approaches, it has recently been proposed to employ reinforcement learning for finding the optimal policy, while in a secondary step a control method is used to certify safety and, if necessary, adapt the applied action \cite{Fisac2019}. This approach can be realized using control barrier functions \cite{Taylor2019}, or a "backup" policy, which is locally safe in some region of the state space~\cite{Bastani2021}. While designing control barrier functions is challenging in general, determining locally safe policies often requires solving computationally expensive optimization problems on-line \cite{Bastani2021} or can only be applied to linearized systems \cite{Wabersich2021b}. Therefore, the practical applicability of such safety filters in combination with highly flexible RL techniques is currently limited.\looseness=-1

\looseness -1 \paragraph*{Our contributions} We mitigate these weaknesses by proposing {\em hallucinating upper confidence safety filters} for ensuring safety of arbitrary policies applied to stochastic, nonlinear systems for which merely a model with high probability error bounds is known. To this end, we first establish a {\em relationship between state constraints and level sets of value functions}. Using the concept of hallucinating inputs~\cite{Curi2020}, we show that these value functions can be efficiently estimated with standard reinforcement learning methods. Our approach can be naturally extended to {\em finding safe policies}, by formulating it as robust reinforcement learning problem. These safe policies can then be used for {\em computationally efficient on-line safety adaptation} of arbitrary reinforcement learning policies. We demonstrate the effectiveness of the proposed method on deep RL benchmark tasks.\looseness=-1

The remainder of this paper is structured as follows. In \cref{sec:problem}, we formalize the problem setting. The hallucinating upper confidence filter is explained and shown to yield safe policies in \cref{sec:hucf}. Finally, the performance of the safety filter is demonstrated in simulations in \cref{sec:eval}, before the paper is concluded in \cref{sec:conc}.

\section{PROBLEM STATEMENT AND BACKGROUND}\label{sec:problem}

We consider a discrete-time dynamical system
\begin{align}\label{eq:true_sys}
    \statet{k+1}=\dynamics(\statet{k},\actiont{k})+\noiset{k},
\end{align}
where $\statet{k}\in\stateset\subset\mathbb{R}^{d_x}$ are states, $\actiont{k}\in\actionset\subset\mathbb{R}^{d_u}$ control actions, $\noiset{k}\!\sim\!\noisedist$ is process noise sampled from a zero-mean probability distribution $\noisedist$\footnote{We consider constant noise distributions $\rho$ for notational simplicity, but our approach directly extends to state and action dependent distributions.}, and $\dynamics\!:\stateset\times\actionset\rightarrow\stateset$ denotes the unknown deterministic transition function. 
The control actions $\actiont{k}$ are determined using a policy $\policy\!:\stateset\rightarrow\actionset$, with the goal to maximize an expected cumulative return
\begin{subequations}
\begin{align}
    \!\!\!\return(\dynamics,\policy;\state) \!=&\expectation{\noise}{\sum\limits_{k=0}^{\infty}\gamma^k\reward(\statet{k},\policy(\statet{k}))},\\
    &\text{s.t. } \statet{k+1} \!=\! \dynamics(\statet{k}, \policy(\statet{k}))\!+\!\noiset{k},~\noiset{k}\!\sim\!\noisedist,\!\!\!\\
    &\quad~~\statet{0}=\state,
\end{align}
\end{subequations}
where 
$r:\stateset\times\actionset\rightarrow\mathbb{R}$ is a known immediate reward function and $\gamma\in(0,1)$ is a discount factor.

In practice, the policy $\policy$ must additionally ensure safety of the closed-loop dynamical system, e.g., because damage to the system described by $\dynamics$ must be avoided. In the RL literature, this is typically addressed through constrained Markov decision processes, which additionally consider a constraint on a cumulative cost function 
\begin{align}\label{eq:cumcost}
	\!C(\dynamics,\policy; \state)\!=\!\expectation{\noise}{\sum_{k=0}^{\infty} \gamma^k c(\statet{k})\!}\!<\!\xi,
\end{align}
where $c:\stateset\rightarrow\mathbb{R}$ is an immediate cost, $\xi\in\mathbb{R}$ is a constant specifying the constraint, and $\statet{k}$ is defined iteratively through \eqref{eq:true_sys} with actions $\actiont{k}=\policy(\statet{k})$ and initial state $\statet{0}=\state$. Therefore, an optimization problem of the form 
\begin{subequations}
\begin{align}\label{eq:opt_pol}
    \policy^*=&\argmax\limits_{\policy}R(\dynamics,\policy;\state)\\
    &~\text{s.t. } C(\dynamics,\policy; \state) < \xi \label{eq:opt_pol_constraint}
\end{align}
\end{subequations}
is usually solved to determine safe policies. 

While this problem can be directly solved by adapting standard RL algorithms with techniques akin to Lagrangian relaxation \cite{Paternain2019}, this approach generally cannot ensure safety {\em during} training. 
Moreover, it does not reflect the fact that the safety of many systems is defined in terms of safe and unsafe states classified into a set of safe states $\safeset\subset\stateset$ and its complement $\unsafeset=\stateset\setminus\safeset$. 
For example, an autonomously driving car should not leave the road, which directly defines the road as $\safeset$. 
When using the natural indicator $\bm{1}_{\bm{x}\in\unsafeset}$ as cost function, satisfying \Cref{eq:opt_pol_constraint} bounds the discounted probability of violating the constraints by $\xi$. 
Nonetheless, this does not guarantee that constraints will not be violated when deploying $\policy^*$.\looseness=-1

Therefore, we consider safety in terms of state constraints $\statet{k}\in\safeset$, which we require to hold with high probability, since the process noise $\noise$ generally prevents deterministic guarantees. This leads to the following definition of safety.\looseness=-1
\begin{definition}
A policy $\policy$ is $K$-step $\delta$-safe for a state $\state\in\stateset$ if it holds that $\probability{\statet{k}\in\safeset~\forall k=0,\ldots,K|\statet{0}=\state}\!\geq\! 1-\delta$, where states $\statet{k}$ are defined in \eqref{eq:true_sys}.
\end{definition}
The concept of $K$-step $\delta$-safety is commonly found in stochastic model predictive control, where it is typically referred to as joint chance constraint \cite{Mesbah2016}. 
\begin{remark}
We consider finite values of $K$ because ensuring $\delta$-safety over an infinite horizon, i.e., $K=\infty$, is not possible for unbounded process noise $\noise$ in general. This can be easily seen for a system with $\dynamics=\bm{0}$ and i.i.d. zero mean Gaussian noise  $\noise$, which almost surely leaves any compact safe set $\safeset$ eventually.
\end{remark}
In order to obtain the optimal policy $\safepolicy^*$ ensuring $\delta$-safety, we generally need to consider the optimization problem
\begin{subequations}\label{eq:safefilt_prob}
\begin{align}
	\!\!\safepolicy^*=&\argmax\limits_{\policy}~R(\dynamics,\policy;\state)\\
	&~\text{s.t. }  \probability{\!\statet{k}\!\!\in\!\safeset~\forall k\!=\!0,\ldots,K|\statet{0}\!=\!\state\!}\!\geq\! 1\!-\!\delta.\!\label{eq:chanceConst}
\end{align}
\end{subequations}
Solving this optimization problem is challenging since there usually exists no closed-form expression for the probability \eqref{eq:chanceConst}, such that computationally expensive uncertainty propagation methods have to be employed, e.g., generalized polynomial chaos expansions \cite{Kim2013}. 

In order to efficiently determine approximate solutions for~\eqref{eq:safefilt_prob}, we follow the idea of \cite{Wabersich2021b} and separate it into two phases: an initial phase for determining a nominal policy $\policy^*$ using an arbitrary method, followed by an on-line phase in which a safety filter is employed to adapt the policy $\policy^*$ to ensure $K$-step $\delta$-safety. Since we cannot ensure safety  without any knowledge about $\dynamics$, we assume to have access to a set of {\em plausible} models $\mathcal{M} = \left\{ \bm{f} \mid |\bm{f}-\model| \leq \beta\bm{\sigma} \right\}$ described by a nominal model $\model\!:\stateset\!\times\actionset\!\rightarrow\!\mathbb{R}^{d_x}$, the state-action dependent
uncertainty about the model $\bm{\sigma}\!:\stateset\!\times\!\actionset\!\rightarrow\!\mathbb{R}^{d_x}$, and a constant scaling factor $\beta\!\in\!\mathbb{R}_+$. 
We assume that this set of models is {\em well-calibrated}, i.e., $\dynamics\! \in\! \mathcal{M}$ with high probability, as formalized in the following.
\begin{assumption}[\cite{Curi2020}]\label{ass:calli}
    The statistical model is calibrated with respect to $\dynamics$, i.e., there exists a $\beta\in\mathbb{R}_+$ such that, with probability at least $1-\delta_f$, it holds jointly for all $\state,\action\in\stateset\times\actionset$ that $|\dynamics(\state,\action)-\model(\state,\action)|\leq
    \beta \bm{\sigma}(\state,\action)$, element-wise.
\end{assumption}
Since the statistical model is often obtained by applying supervised machine learning
to data obtained from policy roll-outs \cite{Curi2020}, the uncertainty usually decreases with the number of roll-outs. Thereby, this assumption typically enables less conservative and higher performant policies over time.\looseness=-1 

Using \cref{ass:calli}, we investigate the following sub-problems for the derivation of the safety filter.
\paragraph{State Constraints as Cumulative Cost}  In order to enable the application of reinforcement learning methods, we consider the problem of converting the $K$-step $\delta$-safety constraint \eqref{eq:chanceConst} for known dynamics $\dynamics$ into a constraint on an expected cumulative cost function. We show that this can be achieved by deriving a condition of the form
\begin{align}\label{eq:constraint2cost}
	\!\expectation{\noise}{C(\dynamics,\policy; \dynamics(\state,\policy)+\noise)}\!<\!\xi.
\end{align}
for suitably chosen immediate costs $c$, cf. \cref{sec:constraint2cost}.
\paragraph{Safety Filter} Using this condition, we derive a novel approach for computing safe policies $\safepolicy$ for systems with unknown dynamics $\dynamics$. This allows us to address the problem of ensuring the safety of a possibly unsafe nominal policy $\policy^*$ on-line using a confidence-based filter
\begin{subequations}\label{eq:safefilt_plausible}
\begin{align}
        \hat{\policy}(\state)=&\argmin\limits_{\action\in\actionset} \|\policy^*(\state)-\action\|,\qquad\\
        &~\text{s.t. } \max_{\bm{f} \in \mathcal{M}} \expectation{\noise}{C(\bm{f},\safepolicy; \state')} < \xi,\\
    &\qquad\state' = \bm{f}(\state, \action)+\noise,~\noise\sim\noisedist,    \end{align}
\end{subequations}
which outputs the closest safe action to $\policy^*$.
We derive tractable formulations for this filter in \cref{sec:hucf}.\looseness=-1



\section{EXPRESSING STATE CONSTRAINTS THROUGH COST FUNCTIONS}\label{sec:constraint2cost}

To reformulate the $\delta$-safety constraint into a constraint on cumulative costs, we first show in \cref{subsec:cost_sublevel} that sub-level sets of $C$ contained in $\safeset$ can be easily defined. Based on this result, we derive sufficient conditions on the cost function, which allow to conclude safety from cumulative cost constraints in \cref{subsec:forward_invar}, providing useful design freedom.\looseness=-1

\subsection{Safe Sub-Level Sets of the Cumulative Cost}\label{subsec:cost_sublevel}

For deriving the sub-level set $\levelset{\bar{\xi}}\!=\!\{\state\!\in\!\stateset\!: C(\dynamics,\policy;\state)\!<\! \bar{\xi} \}$, $\bar{\xi}\!\in\!\mathbb{R}$, which is contained in the set of safe states $\safeset$, we consider an immediate cost function  $c:\stateset\rightarrow\mathbb{R}$ satisfying 
\begin{align}\label{eq:tildec}
\underline{c} \leq c(\state)\leq \bar{c}  \quad\forall\state\in\stateset, \quad&c(\state)\geq \hat{c}  \quad\text{if } \state\in\unsafeset
\end{align}
for constants $\underline{c}, \bar{c},\hat{c}\!\in\!\mathbb{R}$. For example, using the indicator function $\bm{1}_{\bm{x}\in\unsafeset}$ as cost, which equals $1$ for $\bm{x}\in\unsafeset$ and $0$ otherwise, implies $\underline{c}\!=\!0$ and $\bar{c}\!=\!\hat{c}\!=\!1$. 
Using this definition, we can define an inner-approximation of the safe set of states $\safeset$ through the $\hat{c}$ sub-level set of the immediate cost $c$, which becomes exact if $c(\state)\!<\!\hat{c}$ for all $\state\!\in\!\safeset$. Moreover, we can define the expected cumulative cost using \eqref{eq:cumcost}.\looseness=-1

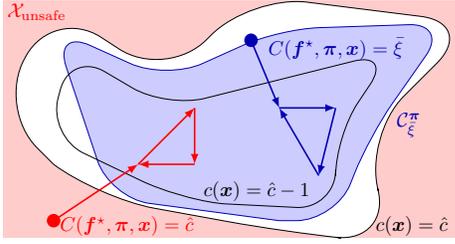
\begin{figure}
    \centering
    \scalebox{0.75}{
    \begin{tikzpicture}
    \fill[red!20] (-0.9,0.2) rectangle (7.2,4.4);
    \node[red] at (-0.3,4.2) {$\unsafeset$};
    
    \fill[rounded corners=1cm, white, draw=black] (0,1.5) -- (-1,2.5) -- (0,4.3) -- (2,3) -- (3.5,4.5)  -- (7.5,4) -- (5.5,2.0) -- (6,0.1) -- (4.1,0.3) -- (0.8,0.9) -- (0,1.5);
    \node at (6.35,0.4) {$c(\state)=\hat{c}$};
    \draw[blue!70!black, rounded corners=0.7cm, fill=blue!20] (4,0.5) -- (5,0.5) -- (5.6,2) -- (7.0,4) -- (4,3.9) -- (2,3.) -- (0,3.9) -- (1,0.9) -- (4,0.5);

    \draw[rounded corners=0.7cm] (3,0.8) -- (5.,0.7) -- (5,2) -- (6,3.5) -- (2,2.5) -- (0.3,3.5) -- (-0.1,2) -- (2.2,0.9) -- (3,0.8);
    \node at (3.6,1.02) {$c(\state)=\hat{c}-1$};
    
    \draw[red, thick, -latex] (0,0.5) -- (1.5,1.5);
    \draw[red, thick, -latex] (1.5,1.5) -- (2.5,2.5);
    \draw[red, thick, -latex] (2.5,2.5) -- (2.5,1.5);
    \draw[red, thick, -latex] (2.5,1.5) -- (1.5,1.5);
    \fill[red] (0,0.5) circle (0.12cm);
    \node[red] at (1.3,0.4) {$C(\dynamics,\policy,\state)=\hat{c}$};
    
    \draw[blue!70!black, thick, -latex] (3.5,3.7) -- (4,2.5);
    \draw[blue!70!black, thick, -latex] (4,2.5) -- (5,2.5);
    \draw[blue!70!black, thick, -latex] (5,2.5) -- (4.7,1.3);
    \draw[blue!70!black, thick, -latex] (4.7,1.3) -- (4,2.5);
    \fill[blue!70!black] (3.5,3.7) circle (0.12cm);
    \node[blue!70!black] at (5.0,3.55) {$C(\dynamics,\policy,\state)=\bar{\xi}$};
    \node[blue!70!black] at (6.3,2.2) {$\levelset{\bar{\xi}}$};
    
    \end{tikzpicture}}
    \vspace{-0.3cm}
    \caption{The expected cumulative cost can be $0$ even if the immediate cost $c$ at the first state is greater than $0$, as this positive cost can be compensated by negative costs afterwards (red trajectory). Therefore, $\levelset{\hat{c}}\not\subset\safeset$, such that we have to consider the tightened threshold $\bar{\xi}$, which ensures that states $\state \in \levelset{\bar{\xi}}$ start with immediate cost $c(\state)\leq \hat{c}$ (blue trajectory).}
    \label{fig:levelset}
\end{figure}

While one might think that the definition of the immediate cost $c$ in \eqref{eq:tildec} ensures that the $\hat{c}$ sub-level set $\levelset{\hat{c}}$ of $C$ is also contained in the safe set of states $\safeset$, this is not true in general. As illustrated by the red trajectory in \cref{fig:levelset}, the cumulative cost $C$ can equal $\hat{c}$ even if the immediate cost $c$ in the initial state is greater than $\hat{c}$, since negative costs of following states along the trajectory can compensate it. Therefore, the sub-level set $\levelset{\hat{c}}$ is generally not completely contained in the set of safe states $\safeset$, such that we must consider a tightened threshold $\bar{\xi}$. Due to the lower bound $\underline{c}$ of the cost $c$, this constant $\bar{\xi}$ can be determined using the following lemma.\looseness=-1
\begin{lemma}\label{lem:levelset}
Consider an immediate cost function $c\!:\stateset\rightarrow\mathbb{R}$ satisfying \eqref{eq:tildec}. Then, it holds that $\levelset{ \bar{\xi}}\subset\stateset_{\mathrm{safe}}$, where
\begin{align}\label{eq:xi1}
    \bar{\xi}= \gamma\min\limits_{\state\in\stateset}C(\dynamics,\policy;\state)+\hat{c}.
\end{align}
\end{lemma}
\begin{proof}
Due to the lower bound for $c$, $C$ is lower bounded, such that we obtain $C(\dynamics,\policy,\state)\geq c(\state)+\gamma C_{\min}$ from Bellman equation, where $C_{\min}=\min_{\state\in\stateset}C(\dynamics,\policy;\state)\geq \frac{\underline{c}}{1-\gamma}$. Moreover, due to condition \eqref{eq:tildec} we have $c(\state)>\hat{c}$ for $\unsafeset$, which yields $C(\dynamics,\policy,\state)> \hat{c}+\gamma C_{\min}$ for all $\state\in\unsafeset$. Therefore, the level set $\levelset{\bar{\xi}}$ is completely contained in $\safeset$, i.e., $\levelset{\bar{\xi}}\subset\stateset_{\mathrm{safe}}$, which concludes the proof.
\end{proof}
This lemma relies on the idea that the cumulative cost can be lower bounded by $\min_{\state\in\stateset}C(\dynamics,\policy;\state)$, such that any state with immediate cost $c$ greater than $\hat{c}$ also must have an expected cumulative cost greater than $\bar{\xi}$. 
For the example of the indicator cost, $\bar{\xi}$ can be straightforwardly computed as $\bar{\xi}\!=\!1$ since $C$ is trivially lower bounded by $0$. It is straightforward to see that this choice of cost function generally allows to accurately approximate $\safeset$ using $\levelset{\bar{\xi}}$, and indeed $\safeset\!=\!\levelset{\bar{\xi}}$ is possible for deterministic dynamics with $\noise\!=\!\bm{0}$. However, \cref{lem:levelset} is not limited to indicator type cost functions, but applies to arbitrary costs $c$ satisfying \eqref{eq:tildec}. This is particularly beneficial for computing optimal policies using $C$, where informative gradients may aid the convergence of common RL techniques. Thus, \cref{lem:levelset} allows a flexible approximation of the safe set $\safeset$ suitable for the optimization-based approaches employed in the following sections.\looseness=-1


\subsection{Cumulative Cost Safety Conditions}\label{subsec:forward_invar}

In order to express $K$-step $\delta$-safety through expected cumulative costs $C$, it remains to derive conditions which ensure that the system state $\statet{k}$ stays inside the sub-level set $\levelset{\bar{\xi}}$ for all $k=1,\ldots,K$ with probability $\delta$. For this purpose, we employ techniques from stochastic stability analysis \cite{Li2013}. 
In particular, we define the cost-value function of a policy at a given state $\state$ as $V_{\policy}(\state) \equiv C(\dynamics\!,\policy;\state)$. 

\begin{theorem}\label{thm:const2cost}
Consider an immediate cost function $c:\stateset\rightarrow\mathbb{R}$, which satisfies \eqref{eq:tildec}. Assume there exists a class~$\mathcal{K}$ function\footnote{A function $\alpha: \mathbb{R}_{0,+}\rightarrow\mathbb{R}_{0,+}$ is a class~$\mathcal{K}$ function, if it is monotonically increasing and $\alpha(0)=0$.} $\alpha:\mathbb{R}\rightarrow\mathbb{R}_{0,+}$, such that
\begin{align}
\label{eq:Lyap_cond_reform}
    &\expectation{\noise}{V_{\policy}(\state')}\leq V_{\policy}(\state)\!-\!\alpha(V_{\policy}( \state)\!-\!C_{\min})
\end{align}
holds for all $\state\in\safeset$ with $\state'=\dynamics(\state,\policy(\state))+\noise)$. Then, 
\begin{align}
    \expectation{\noise}{V_{\policy}(\state')}\leq\xi<\bar{\xi}
\end{align}
guarantees that the policy $\policy$ is $K$-step $\delta(\xi)$-safe.
\end{theorem}
\begin{proof}
The result directly follows from \cref{lem:levelset} and \cref{lem:stabBasic}, which ensure $\delta$-safety with $\delta\!=\!\delta_{\mathrm{FL}}(\xi)$.
\end{proof}
Condition \eqref{eq:Lyap_cond_reform} effectively resembles a Foster-Lyapunov drift condition, which is satisfied if stochastic stability can be shown with $V_{\policy}$ as a Lyapunov function \cite{Meyn1993}. Since  stability is a well-studied problem, it has been shown that this condition can be satisfied for many dynamics $\dynamics$, e.g., systems which are asymptotically controllable with respect to the immediate cost $c$ \cite{Gaitsgory2018}. 
In contrast to stability theory, \cref{thm:const2cost} does require $C$ to be positive definite or the existence of a class $\mathcal{K}$ function lower bounding $C$.
Therefore, the conditions of \cref{thm:const2cost} are slightly weaker than for stability.\looseness=-1

Due to the close relationship to stability, it is straightforward to see that the increase rate of $\alpha$ determines the convergence rate of the system. If $\alpha$ is only slowly growing, a relatively small noise realization can cause an increase in the expected cumulative cost, and thereby, increases the probability $\delta$ of leaving the safe set. This can be compensated by choosing a smaller value of $\xi$, such that there essentially is a larger margin between the safe initial states $\statet{0}$ and the unsafe set $\unsafeset$. Note that the noise distribution also affects the probability $\delta$ through \eqref{eq:Lyap_cond_reform}, since flat distributions with heavy tails generally cause higher values of $\expectation{\noise}{V_{\policy}(\state)}$ leading to smaller increase rates of $\alpha$.


\section{HALLUCINATING UPPER CONFIDENCE SAFETY FILTERS}\label{sec:hucf}

We now derive a tractable safety filter for unknown dynamics, for which merely a set of plausible models is available. To this end, we first show how the safety analysis of \cref{sec:constraint2cost} can be extended to unknown dynamics by reformulating it as reinforcement learning problem in \cref{subsec:uc-trick}. Based on this reformulation, we present a natural approach to obtain backup policies for the safety filter by computing a {\em safe policy} through robust reinforcement learning in \cref{subsec:max_safe}. Finally, the pre-computed backup policy is employed in a confidence-based safety filter for ensuring constraint satisfaction in \cref{subsec:safe_filt}.\looseness=-1

\subsection{Safety Certification with Unknown Dynamics}\label{subsec:uc-trick}

Since we assume only the availability of a set of plausible models $\mathcal{M}$, but not the true dynamics $\dynamics$, we cannot determine $C$ and consequently cannot directly exploit \Cref{thm:const2cost} for determining a safe policy. 
To overcome this issue, we must be \emph{pessimistic} about the dynamics. 
In particular, we define the \emph{pessimistic cost-value} as
\begin{align} \label{eq:pessimistic_costvalue}
    V_{\policy}^{(p)}(\state) \equiv \max_{\bm{f} \in \mathcal{M}} C(\bm{f},\policy;\state).
\end{align}
It is straightforward to see that $V_{\policy}^{(p)}(\state) \leq \xi \Rightarrow V_{\policy}(\state) \leq \xi$ due to \cref{ass:calli}. 
Hence, the technique for unknown models reduces to solving \eqref{eq:pessimistic_costvalue} and verifying if its value satisfies the conditions of \Cref{thm:const2cost}. 
To solve the pessimistic verification problem \eqref{eq:pessimistic_costvalue} we reparameterize the functions $\bm{f}\in\mathcal{M}$ following \cite{Curi2020} as 
\begin{align}\label{eq:reparam}
\bm{f}(\state,\action)=\model(\state,\action)+\ucbound,
\end{align}
where $\hpolicy:\stateset\times\actionset\!\rightarrow\![-1,1]^{d_x}\!$,  $\bm{\Sigma}(\state,\action)\!=\!\mathrm{diag}(\bm{\sigma}(\state,\action))$ and $\model$ is the model mean function. We refer to $\hpolicy$ as the hallucinating policy in the sequel since it acts on the outputs of the dynamics similarly as a policy acts on its inputs.
Moreover, the set of admissible functions $\hpolicy$ is defined via
\begin{align}
    \mathcal{N}&=\{\hpolicy:~ -1\leq \hpolicy(\state,\action)\leq 1~~ \forall \state,\action\in\stateset\times\actionset \}.
\end{align}

To determine that a policy is safe, we must verify that the value of the pessimistic estimate $V_{\policy}^{(p)}$ defined as
\begin{subequations}
\begin{align}
    \!\!\!V_{\policy}^{(p)}(\state)  =& \max_{\hpolicy \in \mathcal{N}} \mathbb{E}_{\noise}\Bigg[ \sum_{k=0}^\infty \gamma^k c(\statet{k})\Bigg] \\
    &\text{s.t. } \statet{k + 1} = \bm{f}(\statet{k},\policy(\statet{k})) + \noiset{k},~  \noiset{k} \sim  \noisedist, \label{eq:ucb_model}\\
    &\quad~~\bm{f}(\state,\action) \!=\!\model(\state,\action)\!+\!\ucbound,\!\!\label{eq:reparam_model} \\
    \!\!\!&\quad~~\statet{0}\!=\!\state 
\end{align}
\end{subequations}
satisfies the conditions of \Cref{thm:const2cost}. 
Computing $V_{\policy}^{(p)}$ is equivalent to an RL problem that can be solved using standard RL algorithms, where the policy is $\hpolicy$ and the dynamics is given by \eqref{eq:reparam}, i.e., it can be done purely in simulation. 
Using this formulation as an optimization of the hallucinating policy $\hpolicy$, it is straightforward to extend \cref{thm:const2cost} to unknown dynamics $\dynamics$ as shown in the following proposition.
\begin{proposition}\label{lem:safe_cert}
Consider a set of plausible models $\mathcal{M}$ satisfying \cref{ass:calli} and an immediate cost $c$,  which satisfies~\eqref{eq:tildec}. If $V_{\policy}^{(p)}(\state)$ satisfies 
\begin{align}\label{eq:pessLF}
    &\max_{\bm{\eta}\in\mathcal{N}}\expectation{\noise}{V_{\policy}^{(p)}(\state')}\leq
    V_{\policy}^{(p)}(\state)\!-\!\alpha(V_{\policy}^{(p)}(\state)\!-\!C_{\min})
\end{align}
    for all $\state\!\in\!\safeset$, where $\state'$ is the next state defined through the reparameterized dynamics \eqref{eq:ucb_model}, \eqref{eq:reparam_model}, for all $\state\!\in\!\safeset$, then, $\max_{\bm{\eta}\in\mathcal{N}}\!\mathbf{E}_{\noise}[V_{\policy}^{(p)}(\state')]\!\leq\! \xi\!<\!\bar{\xi}$
guarantees $K$-step $\delta$-safety of~$\policy$.
\end{proposition}
\begin{proof}
Due to Assumption~\ref{ass:calli}, $V_{\policy}^{(p)}(\state) \!\leq\! \bar{\xi}$ implies $V_{\policy}(\state) \!\leq\! \bar{\xi}$
with probability at least $1-\delta_f$. Therefore, $\max_{\bm{\eta}\in\mathcal{N}}\hat{V}(\state')\leq \bar{\xi}$ implies that $\state\in\safeset$ due to \cref{lem:levelset}. 
Moreover, it can be directly shown that
\begin{align*}
    &\max_{\bm{\eta}\in\mathcal{N}}\expectation{\noise}{V_{\policy}^{(p)}(\dynamics(\state,\policy(\state))+\noise)}\leq \max_{\bm{\eta}\in\mathcal{N}}\expectation{\noise}{V_{\policy}^{(p)}(\state')}
\end{align*}
for $\state'$ the next state defined through the reparameterized dynamics \eqref{eq:ucb_model}, \eqref{eq:reparam_model} with probability at least $(1-\delta_f)$ due to Assumption~\ref{ass:calli}. 
Due to \eqref{eq:pessLF}, this yields
\begin{align*}
    \max_{\bm{\eta}\in\mathcal{N}}&\expectation{\noise}{V_{\policy}^{(p)}(\dynamics(\state,\policy(\state))+\noise)}\leq\\ &\qquad\qquad\qquad\qquad\quad V_{\policy}^{(p)}(\state)-\alpha(V_{\policy}^{(p)}(\state)-C_{\min}),
\end{align*}
such that
we can apply \cref{lem:stabBasic}. Hence safety follows with $\delta=\delta_{\mathrm{FL}}(\xi)+\delta_f-\delta_{\mathrm{FL}}(\xi)\delta_f$.
\end{proof}
When the model is known accurately, i.e., $\bm{\sigma}(\state,\action)=\bm{0}$ for all $\state,\action\in\stateset\times\actionset$, the conditions of \cref{lem:safe_cert} intuitively reduce to the conditions of \cref{thm:const2cost}. 



\subsection{Robust Reinforcement Learning of Safe Policies}\label{subsec:max_safe}
Based on the formulation of the $K$-step $\delta$-safety as an optimization problem in \cref{lem:safe_cert}, it is natural to augment the optimization problem to directly find $\delta$-safe policies by finding the policy that minimizes the pessimistic cost estimate $V_{\policy}^{(p)}$. 
Namely, we propose to obtain safe policies via\looseness=-1
%
\begin{align}\label{eq:maxSafePi}
    \safepolicy \coloneqq \argmin\limits_{\policy\in\Pi}  V_{\policy}^{(p)} = \argmin\limits_{\policy\in\Pi} \max_{\hpolicy\in \mathcal{N}}\expectation{\state}{C(\bm{f},\policy;\state)}
\end{align}
where the dynamics $\bm{f}$ is the reparameterized dynamics in \eqref{eq:reparam}.
As \eqref{eq:maxSafePi} can be seen as a robust RL problem, we refer 
to $\safepolicy$ as the \emph{learned safe policy} in the following.
To solve \eqref{eq:maxSafePi} we use standard robust RL techniques such as those in \cite{Curi2021,Pinto2017}, which perform gradient descent for $\policy$ and gradient ascent for~$\hpolicy$. 
%
%
Moreover, if the cost $c$ and the discount $\gamma$ allow to establish the safety of this system for some policy $\policy$, it is straightforward to show that \eqref{eq:maxSafePi} yields a $\delta$-safe policy.
\begin{proposition}\label{prop:maxSafe}
Assume the learned safe policy $\safepolicy$ is obtained using \eqref{eq:maxSafePi} with sufficiently expressive function classes $\Pi$, $\mathcal{N}$, and has a unique solution. 
If there exists a policy $\policy$ and a class $\mathcal{K}$ function $\alpha$ such that $V_{\policy}^{(p)}$ satisfies \eqref{eq:pessLF}, $\safepolicy$ is $K$-step $\delta$-safe for all $\state\!\in\!\stateset$ with $V_{\safepolicy}^{(p)}(\state)\!\leq\! \xi\!<\!\bar{\xi}$.\looseness=-1
\end{proposition}
\begin{proof}
Given the policy $\policy_{\state}$ defined as\looseness=-1
\begin{align*}
    \policy_{\state} &= \argmin\limits_{\policy\in\Pi} \max_{\hpolicy\in \mathcal{N}}C(\bm{f},\policy;\state),
\end{align*}
we can lower bound the expected cost for arbitrary $\policy\in\Pi$ by 
$\mathbf{E}_{\state}[V_{\policy}^{(p)}]\geq \mathbf{E}_{\state}[V_{\policy_x}^{(p)}]$. Therefore, for a sufficiently expressive policy class $\Pi$, it must follow that $\policy_x=\argmin_{\policy\in\Pi}\mathbf{E}_{\state}[V_{\policy}^{(p)}]$ holds. Moreover, we have  $\mathbf{E}_{\state}[V_{\policy_x}^{(p)}]\geq \mathbf{E}_{\state}[C(\bm{f},\policy_x;\state)]$ for arbitrary $\hpolicy\in\mathcal{N}$. Therefore, for a sufficiently expressive function class $\mathcal{N}$, it must hold that $\max_{\hpolicy\in\mathcal{N}}\mathbf{E}_{\state}[C(\bm{f},\policy_x;\state)]=\mathbf{E}_{\state}[V_{\policy_x}^{(p)}]$, such that $\safepolicy=\policy_x$ is ensured due to uniqueness of the solution. Since $V_{\safepolicy}^{(p)}(\state)\leq V_{\policy}^{(p)}(\state)$ is satisfied due to point-wise optimality of $\policy_x$ and \eqref{eq:pessLF} can be straightforwardly reformulated to 
\begin{align*}
    (1-\gamma)V_{\policy}^{(p)}(\state)+\gamma\alpha(V_{\policy}^{(p)}( \state)-C_{\min}) \leq c(\state),
\end{align*}
safety of $\safepolicy$ immediately follows from the existence of a $\delta$-safe policy and \cref{lem:safe_cert}.
\end{proof}
Since there exist combinations of dynamics $\dynamics$ and safe sets $\safeset$ for which safety cannot be ensured, \cref{prop:maxSafe} cannot guarantee the $\safepolicy$ to be always $K$-step $\delta$-safe. However, as discussed in \cref{subsec:forward_invar}, there exist system classes for which \eqref{eq:Lyap_cond_reform} can be satisfied. 
Moreover, the necessary function classes for \eqref{eq:maxSafePi} are well-known for many systems, e.g., it is straightforward to see that continuous systems and costs require piece-wise continuous policies in general.

\subsection{Ensuring Constraint Satisfaction with Safety Filters}\label{subsec:safe_filt}

While the learned safe policy $\safepolicy$ is safe during a policy roll-out under certain assumptions, it can possibly result in bad performance since it does not consider the reward function~$r$. In contrast, the nominal policy $\policy^*$ \eqref{eq:opt_pol} results in a high reward of generated trajectories, but can possibly lead to unsafe states. Therefore, we ideally want to maintain the beneficial properties of both policies, while avoiding their shortcomings. The core idea for achieving this relies on a continuous monitoring of every nominal action $\policy^*(\state)$, such that they can be adapted to ensure a safe roll-out of $\safepolicy$ afterwards. Using the reparameterization \eqref{eq:reparam} of the set of plausible models $\mathcal{M}$, this yields our {\em confidence-based safety filter}\looseness=-1
\begin{subequations}\label{eq:safety_filter_const}
\begin{align}
        \!\!\!\hat{\policy}(\state)=
        &\argmin\limits_{\action\in\actionset} \|\policy(\state)-\action\|,\qquad\\
        &\text{s.t. } \max\limits_{\hpolicyaction\in [-1,1]^{d_x}}\expectation{\noise}{ V_{\safepolicy}^{(p)}(\state')} \!\leq\! \xi,\!\!\label{eq:safety_filt_FLcond}\\
    &\quad~~~\state' = \model(\state, \action)\!+\!\ucboundn\!+\!\noise,~\noise\!\sim\!\noisedist.\!
\end{align}
\end{subequations}
Since it cannot be ensured that the state $\statet{k}$ satisfies $V_{\safepolicy}^{(p)}(\statet{k})\leq\xi$ for all $k$, a recovery mechanism steering the system back into this sub-level set is required. This can be straightforwardly achieved using the learned safe policy $\safepolicy$, resulting in the overall roll-out policy
\begin{align}\label{eq:comb_pol}
    \tilde{\policy}(\state) = \begin{cases} 
    \hat{\policy}(\state) &\text{if } V_{\safepolicy}^{(p)}(\state)\leq\xi\\ 
    \safepolicy(\state) &\text{if } V_{\safepolicy}^{(p)}(\state)>\xi
    \end{cases}.
\end{align}
Due to its strong foundation on the learned safe policy $\safepolicy$, the roll-out policy $\tilde{\policy}$ inherits its theoretical safety guarantees as shown in the following theorem.
\begin{theorem}\label{th:safetyFilt}
Consider a set of plausible models $\mathcal{M}$ satisfying \cref{ass:calli} and assume that the learned safe policy $\safepolicy$ satisfies the conditions of \cref{lem:safe_cert}. Then, the confidence-based safety filtered policy \eqref{eq:comb_pol} is $K$-step $\delta$-safe for all states $\state\in\stateset$ with $V_{\safepolicy}^{(p)}(\state)\leq \xi<\bar{\xi}$.
\end{theorem}
\begin{proof}
Since $\safepolicy$ satisfies the conditions of \cref{lem:safe_cert}, the trivial solution $\action=\safepolicy(\state)$ is guaranteed to ensure \eqref{eq:safety_filt_FLcond}. Therefore, \eqref{eq:safety_filter_const} is feasible for all states $\state\in\stateset$ with $V_{\safepolicy}^{(p)}(\state)\leq \xi$, such that the $K$-step $\delta$-safety of $\tilde{\policy}$ follows directly from \cref{lem:safe_cert}.
\end{proof}
While the safety filter problem \eqref{eq:safety_filter_const} is not compatible with standard reinforcement learning methods, it can easily be solved on-line in the fashion of model predictive control. 
In order to see this, note that the pessimistic cost-value $V_{\safepolicy}^{(p)}(\state)$ can be efficiently obtained offline using actor-critic methods for reinforcement learning similar to $V_{\policy}^{(p)}$ in \cref{subsec:uc-trick}. Hence, \eqref{eq:safety_filter_const} requires optimization merely for one time step and consequently only for a single actual and hallucinating adversarial action in contrast to similar predictive safety filter approaches \cite{Bastani2021, Wabersich2021b}, which require optimization over a sequence of actions. Therefore, \eqref{eq:safety_filter_const} can be solved with comparatively low computational complexity using numerical optimization schemes, which allows a straightforward on-line application as safety filter.

\begin{remark}\label{rem:xi2tun}
For practical implementation of \eqref{eq:safety_filter_const}, $\xi$ can be considered a tuning parameter. The smaller its value is, the higher the probability of safety is. However, a small $\xi$ will lead to more conservatism of the safety filter, such that it must be carefully chosen to trade-off safety and performance. 
\end{remark}

\section{EXPERIMENTAL RESULTS}\label{sec:eval}
\setlength{\floatsep}{12pt}
\setlength{\textfloatsep}{12pt}
In this section, we evaluate the safety filter and compare it with three competing algorithms: the constraint-free model-free algorithm SAC \cite{haarnoja2018soft}, a Lagrangian primal-dual approach with SAC as the base algorithm, which we call CMDP \cite{Paternain2019}, and the model-based alternative Safe-CEM \cite{liu2020constrained}.
%
We consider two widely used environments to test our approach. First, we test it on an airplane pitch control \cite{hafner2011reinforcement}, where the pitch angle $\theta$ starts at $-0.2$ radians and the constraint function is simply $c_t = \theta_t$ such that the angle should never exceed $0$. The reward is given by $r_t = -2\theta^2 + 0.02u^2$, where $u$ is the control input. 
Second, we use the Mujoco Half-Cheetah environment with the default reward function \cite{todorov2012mujoco}. 
The constraint is that the forward speed is less than $2$. 
Due to the Cheetah's trot, the penalty is on the average forward speed, calculated as $\bar{v}_{t} = 0.1 v_{t} + 0.9 \bar{v}_{t-1}$, $\bar{v}_0 = 0$, where $v_t$ is the instantaneous speed and $\bar{v}_t$ is the average speed. Thus we use $c_t = \bar{v}_t - 2$.
We run each environment for $100$ episodes, each episode for $1000$ time steps, using $\gamma=0.99$ as a discount factor.\looseness=-1

To learn the model, we use deterministic ensembles of five members following \cite{Curi2020}. 
Each member is a neural network with 3 fully connected layers of width $200$ and Swish non-linearities. 
For the first ten episodes, data is collected using a random policy. 
Such random policy was safe in these environments but only at the given initial conditions, i.e., it is not the learned safe policy used by the safety filter.
After the initial exploration phase, the model is pre-trained for $100$ iterations using Adam with learning rate $0.0005$ and weight decay $0.0001$. 
Then, after each subsequent episode, the model is updated using the additional data collected during the episode.
We store the data using an experience replay buffer of at most $100000$ transitions.
Finally, to solve the safety filter problem \eqref{eq:safety_filter_const} we use the cross-entropy method \cite{botev2013cross} with 1000 particles and 5 iterations per time-step. 

In \Cref{fig:pitch_control}, we show the results in the pitch control environment. In this setting, only the Safety Filter algorithm avoids any constraint violation while achieving comparable performance in terms of returns and costs. 
In \Cref{fig:half_cheetah}, we show the results for the Half-Cheetah. 
Here, both Safe-CEM and the safety filter avoid any constraint violations. 
However, the safety filter achieves higher returns than Safe-CEM.
The main difference between these two environments relies on the backup policy. 
While in the Cheetah it is enough to \emph{do nothing} in order to stop it, in the Pitch Control environment this is not the case and the \emph{learned} safe backup policy is crucial to ensure safety. 
Thus, with these two environments we demonstrate the scalability of our method in the Half Cheetah environment as well as the ability to satisfy constraints in the Pitch Control environment.  

\begin{figure}
    \centering
    \includegraphics[width=\columnwidth]{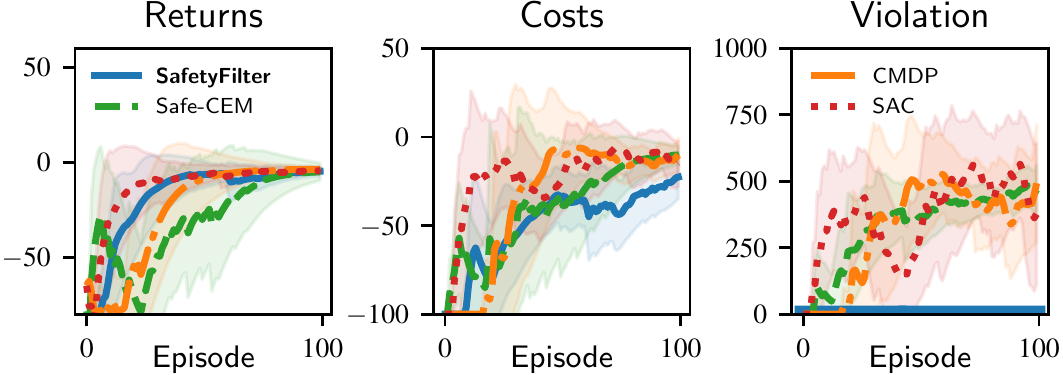}
    \vspace{-0.7cm}
    \caption{Total returns, costs, and constraint violation in the Pitch Control environment. 
    Only the safety filter attains \textbf{no} constraint violations and achieves comparable performance to the benchmarks.
    }
    \label{fig:pitch_control}
\end{figure}

\begin{figure}
    \centering
    \includegraphics[width=\columnwidth]{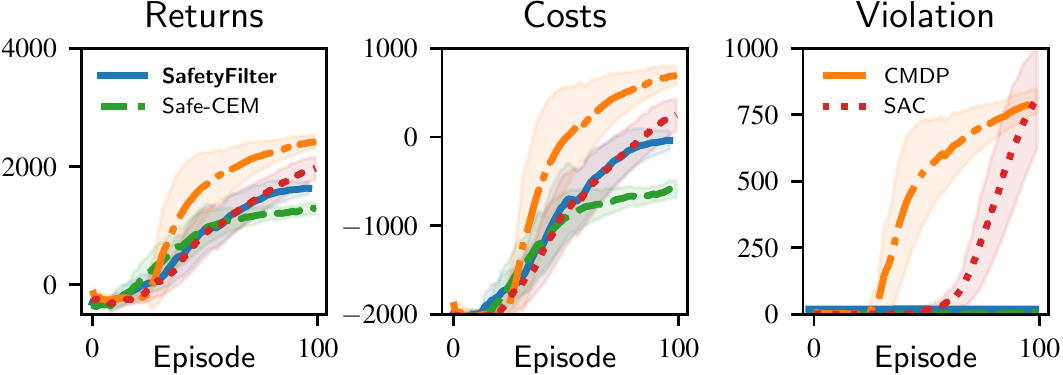}
    \vspace{-0.7cm}
    \caption{
    Total returns, costs, and constraint violation in the Half-Cheetah environment. 
    The safety filter and Safe-CEM achieve \textbf{no} constraint violation.
    These two algorithms perform slightly worse than the benchmarks in terms of returns, but the safety filter performs better than Safe-CEM.}
    \label{fig:half_cheetah}
\end{figure}

\section{CONCLUSION}\label{sec:conc}

In this paper, we introduced confidence-based safety filters, a novel approach for ensuring the safety of nominal policies learnt via standard reinforcement learning techniques. The approach relies on a reformulation of state constraints as cumulative costs, such that safety conditions can be expressed as cost constraints. This allows us to obtain safe policies via robust reinforcement learning, which can be used as ``backup'' policy in a safety filter. We demonstrated  the effectiveness and scalability of our approach in simulations.\looseness=-1




\appendix
\renewcommand{\thelemma}{A.\arabic{lemma}}
\setcounter{lemma}{0}
\renewcommand{\theproposition}{A.\arabic{proposition}}
\setcounter{proposition}{0}
\setlength{\abovedisplayskip}{4pt}
\setlength{\belowdisplayskip}{4pt}
\begin{lemma}\label{lem:constraint}
If there exists a function $V\!:\stateset\!\rightarrow\!\mathbb{R}$ such that\looseness=-1
\begin{align}\label{eq:stabproof3}
    \expectation{\omega}{V(\dynamics(\state,\policy(\state))+\omega)}\leq \theta_1
\end{align}
for $\theta_1\in\mathbb{R}$, then, it holds that 
\begin{align}
    \frac{\theta_2-\theta_1}{\theta_2\!-\!\underline{V}}\leq \probability{\statet{k+1}\in\Vlevelset{\theta_2}|\statet{k}=\state}\leq  \frac{\theta_2-\theta_1}{\bar{V}\!-\!\theta_2}
\end{align}
for every $\underline{V}\leq \theta_1<\theta_2<\bar{V}$, where $\underline{V}=\min_{\state\in\stateset}V(\state)$ and  $\bar{V}=\max_{\state\in\stateset}V(\state)$.
\end{lemma}
\begin{proof}
In order to prove this lemma, we follow the ideas of \cite{Li2013}. It is straightforward to see that
\begin{align}\label{eq:stabproof1}
    &(\theta_2-\underline{V})\probability{\statet{k+1}\notin\Vlevelset{\theta_2}|~\statet{k}=\state}\leq\\ &\qquad\qquad\quad\expectation{\noise_{k}}{\indicator{\statet{k+1}\notin\Vlevelset{\theta_2}}(V(\statet{k+1})-\underline{V})| ~\statet{k}=\state},\nonumber
\end{align}
where $\Vlevelset{\theta_2}=\{\state\in\stateset: V(\state)<\theta_2 \}$, and  $\indicator{\statet{k+1}\notin\Vlevelset{\theta_2}}=1$ if $\statet{k+1}\notin\Vlevelset{\theta_2}$ and $0$ otherwise, since $\indicator{\statet{k+1}\notin\Vlevelset{\theta_2}}V(\statet{k+1})\geq \indicator{\statet{k+1}\notin\Vlevelset{\theta_2}}\theta_2$. Moreover, we trivially have
\begin{align}\label{eq:stabproof2}
    &\expectation{\noise_{k}}{\indicator{\statet{k+1}\notin\Vlevelset{\theta_2}}(V(\statet{k+1})-\underline{V})| ~\statet{k}=\state}\leq\! \\
    &\qquad\qquad\qquad\qquad\qquad\expectation{\noise_{k}}{(V(\statet{k+1})-\underline{V})| ~\statet{k}=\state}.\nonumber
\end{align}
By combining \eqref{eq:stabproof3}, \eqref{eq:stabproof1} and \eqref{eq:stabproof2}, we therefore obtain
$(\theta_2-\underline{V})\probability{\statet{k+1}\notin\Vlevelset{\theta_2}|~\statet{k}=\state}\leq -\underline{V}+\theta_1,$
which results in 
\begin{align}
    \probability{\statet{k+1}\notin\Vlevelset{\theta_2}|~\statet{k}=\state}\leq \frac{\theta_1-\underline{V}}{\theta_2-\underline{V}}.
\end{align}
The proof for the upper bound is analogous.
\end{proof}

\begin{proposition}
\label{lem:stabBasic}
Assume there exists a function $V\!\!:\!\stateset\!\rightarrow\!\mathbb{R}$ and a class $\mathcal{K}$ function $\alpha:\mathbb{R}\rightarrow\mathbb{R}_{0,+}$, such that
\begin{align}
\label{eq:Lyap_cond}
    \expectation{\omega}{V(\dynamics(\state,\policy(\state))+\omega)}-V(\state)\leq -\alpha(V(\state))
\end{align}
holds for all $\state\in\Vlevelset{\bar{\xi}}$ for $\bar{\xi}\in\mathbb{R}$. Then, $\expectation{\omega}{V(\dynamics(\state,\policy(\state))+\omega)}\leq \xi$
with $\xi<\bar{\xi}$ ensures 
\begin{align}
    \probability{V(\statet{k})\leq \bar{\xi}~\forall k=1,\ldots,K|\statet{0}=\state}\geq 1-\delta_{\mathrm{FL}}(\xi)
\end{align}
with
\begin{align}\label{eq:delta2}
    \delta_{\mathrm{FL}}(\xi)=\begin{bmatrix}1&\cdots&0 \end{bmatrix}\! \begin{bmatrix}
    1&\bm{1}^T(\bm{I}-[\bm{P}]_+)\\
    \bm{0}&[\bm{P}]_+
    \end{bmatrix}^{\!K}\!\begin{bmatrix} 0\vspace{-0.2cm}\\ \vdots\vspace{-0.05cm}\\ 1\end{bmatrix}\!\!,\!
\end{align}
where the elements of $\bm{P}$ are defined as
\begin{align}\label{eq:pij}
    \!p_{i,j} \!=\!\!\begin{cases}
    \!\frac{\theta^i\!-\theta^j\!+\alpha(\theta^{j\!+\!1}\!+\!\underline{V})}{\theta^i\!-\underline{V}}
    \!-\!\frac{\theta^{i\!+\!1}\!-\theta^j\!+\alpha(\theta^{j\!+\!1}\!+\underline{V})}{\bar{V}\!-\theta^j\!+\alpha(\theta^{j\!+\!1}\!+\underline{V})}&\!\!\text{if } i\!\leq\! j\\
    \!\frac{(1-\vartheta)\alpha(\theta^{j\!+\!1}+\underline{V})}{\theta^j-\alpha(\theta^{j\!+\!1}+\underline{V})-\underline{V}}&\!\!\text{if } i\!=\!j\!+\!1.\!\!\\
    \!0&\!\!\text{if } i\!>\!j\!+\!1
    \end{cases}\!
\end{align}
and $M$ is the largest integer such that $\theta^1\leq\bar{\xi}$ for $\theta^i$ recursively defined by $\theta^{i-1} = \theta^i+\vartheta\alpha(\theta^i+\underline{V})$
with $\theta^{M+1}$ implicitly defined via $\theta^{M+1}+(\vartheta-1)\alpha(\theta^{M+1}+\underline{V})=\xi$ and sufficiently small $\vartheta\!\in\! (0,1)$.\looseness=-1 
\end{proposition}
\begin{proof}
For proving this proposition, we construct a sequence of sub-level sets $\Vlevelset{\theta^j}$ as illustrated in \cref{fig:contraction} and bound the transition probabilities between them  using \cref{lem:constraint}.
Given a sub-level set $\Vlevelset{\theta^j}$, the probability of transitioning into sub-level set $\Vlevelset{\theta^{j+1}}$ can be lower bounded using $\theta_2=\theta^j-\vartheta\alpha(\theta^{j+1}-\underline{V})$, $\theta_1=\theta^j-\alpha(\theta^{j+1}-\underline{V})$, which yields\looseness=-1
\begin{align*}
    p_{j+1,j}=\frac{(1-\vartheta)\alpha(\theta^{j+1}+\underline{V})}{\theta^j-\alpha(\theta^{j+1}+\underline{V})-C_{\min}}.
\end{align*}
For transitioning from the sub-level set $\Vlevelset{\theta^j}$ to a sub-level set $\Vlevelset{\theta^i}$, $i\leq j$, we have
\begin{align*}
    &\probability{\statet{t+1}\in\Vlevelset{\theta^i}\setminus\Vlevelset{\theta^{i+1}}|\statet{t}\in\Vlevelset{\theta^j} }=\\
    &\quad~\probability{\statet{t+1}\in\Vlevelset{\theta^i}|\statet{t}\in\Vlevelset{\theta^j} }-\probability{\statet{t+1}\in\Vlevelset{\theta^{i+1}}|\statet{t}\in\Vlevelset{\theta^j} }\!,\nonumber
\end{align*}
such that applying \cref{lem:constraint} to both summands with $\theta_2=\theta^i$, $\theta_1=\theta^j-\alpha(\theta^{j+1}-\underline{V})$ and $\theta_2=\theta^{i+1}$, $\theta_1=\theta^j-\alpha(\theta^{j+1}-\underline{V})$, respectively, yields
\begin{align*}
    p_{i,j}\!=\!\frac{\theta^i\!-\!\theta^j\!+\!\alpha(\theta^{j\!+\!1}\!+\!\underline{V})}{\theta^i\!-\underline{V}}
    \!-\!\frac{\theta^{i\!+\!1}\!-\!\theta^j\!+\!\alpha(\theta^{j\!+\!1}\!+\!\underline{V})}{\bar{V}\!-\theta^j\!+\!\alpha(\theta^{j\!+\!1}\!+\underline{V})}.
\end{align*}
Note that for $i=M$ we have $\theta_1=\xi$. 
Since we cannot guarantee to directly transition from sub-level sets $\Vlevelset{\theta^j}$ to sub-level sets $\Vlevelset{\theta^i}$ with $i\geq j+2$, we obtain the trivial bound $p_{i,j}=0$ in this case, which results in \eqref{eq:pij}. Based on the bounds $p_{i,j}$, we can construct a left stochastic matrix similar to the transition matrix of a Markov chain, whose first row corresponds to an absorbing state as shown in \eqref{eq:delta2}. Since the first state is absorbing and the transition probabilities to all other states are lower bounds, multiplying this matrix $K$ times with itself and multiplying the initial probability distribution from the right yields the upper bound $\delta$ for leaving the sub-level set $\Vlevelset{\bar{\xi}}$ within $K$ time steps.
\end{proof}

\begin{figure}
    \centering
    \scalebox{0.75}{
    \begin{tikzpicture}
    
    
    \draw[blue!70!black, rounded corners=0.7cm, fill=blue!30] (4,0.5) -- (5,0.5) -- (5.6,2) -- (7.0,4) -- (4,3.9) -- (2,3.) -- (0,3.9) -- (1,0.9) -- (4,0.5);
    \draw[blue!90!white, rounded corners=0.7cm, fill=blue!20] (4,1.0) -- (5,1.0) -- (5.3,2.5) -- (6.7,3.9) -- (4,3.5) -- (2,2.7) -- (0.5,3.4) -- (1.6,1.2) -- (4,1.0);
    \node[blue!70!black] at (4.6,0.8) {$\levelset{\bar{\xi}}$};
    
    \draw[blue!60!white, rounded corners=0.7cm, fill=blue!10] (3.3,1.3) -- (4.5,1.3) -- (4.9,2.6) -- (6.5,3.7) -- (4,3.2) -- (2.4,2.4) -- (1.1,2.8) -- (2.0,1.5) -- (3.3,1.3);
    \node[blue!70!white] at (4.4,1.2) {$\levelset{\theta^1}$};
    
    \draw[blue!40!white, rounded corners=0.7cm, fill=white] (3.4,1.5) -- (4.1,1.7) -- (4.7,2.9) -- (6.1,3.5) -- (4,3.0) -- (2.8,2.2) -- (1.5,2.6) -- (2.5,1.5) -- (3.4,1.5);
    \node[blue!60!white] at (4.1,1.6) {$\levelset{\theta^2}$};
    \node[blue!40!white] at (3.7,1.9) {$\levelset{\xi}$};
    
    
    \draw[-latex] (1.6,2.4) -- (-1.4,2.4);
    \node at (-0.5,2.9) {$1\!-\!\sum\limits_{i=1}^3p_{i,2}$};
    
    \draw[-latex] (2.0,1.8) -- (2.6,1.8);
    \node at (2.3,2.0) {$p_{3,2}$};
    
    \draw[-latex] (1.9,2.0) -- (1.3,2.0);
    \node at (1.6,2.2) {$p_{1,2}$};
    
    \draw[] (3.1,1.9) to[out=180,in=180,looseness=3] (3.1,1.4);
    \draw[-latex] (3.1,1.9) to[out=0,in=0,looseness=3] (3.1,1.4);
    \node at (3.1,1.2) {$p_{2,2}$};
    
    \end{tikzpicture}}
    \vspace{-0.3cm}
    \caption{In order to certify the $K$-step $\delta$-safety of a policy $\policy$, we define a sequence of sub-level sets $\levelset{\theta^i}$ with decreasing thresholds $\theta^i$. We can bound the probabilities for transitioning to other sub-level sets in each time step using \cref{lem:constraint} as exemplarily illustrated for $\levelset{\epsilon^2}$, such that the probability of leaving the $\levelset{\bar{\xi}}$ can be bounded using methods for Markov chains.\looseness=-1}
    \label{fig:contraction}
\end{figure}
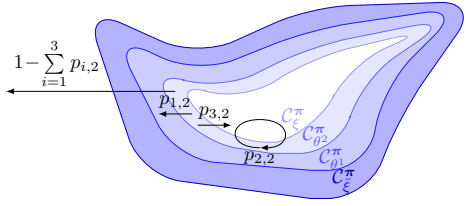






\bibliographystyle{IEEEtran}
\bibliography{IEEEabrv,root.bib}

\end{document}